\documentclass{article}

\usepackage{vmargin}
\setpapersize{A4}

\usepackage[round]{natbib}

\usepackage{amsmath,amssymb,amsthm,graphicx}
\usepackage{url,algorithm,algorithmic}

\newcommand{\inner}[2]{\left\langle #1, #2 \right\rangle}
\newcommand{\g}[1]{\boldsymbol{#1}}

\newcommand{\R}[0]{\mathbb{R}} 
\newcommand{\E}[0]{\mathbb{E}}

\renewcommand{\H}[0]{\mathcal{H}} 
\newcommand{\X}[0]{\mathcal{X}} 
\newcommand{\Y}[0]{\mathcal{Y}} 
\newcommand{\V}[0]{\mathcal{V}} 

\newcommand{\F}[0]{\mathcal{F}} 
\newcommand{\B}[0]{\mathcal{B}} 
 
\renewcommand{\L}[0]{\mathcal{L}} 
\newcommand{\Z}[0]{\mathcal{Z}}

\newcommand{\trace}{\operatorname{Tr}}
\newcommand{\rank}{\operatorname{rank}}
\newtheorem{theorem}{Theorem}
\newtheorem{lemma}{Lemma}

\newtheorem{definition}{Definition}

\begin{document}

\title{Risk Bounds for Learning Multiple Components with Permutation-Invariant Losses}

\author{Fabien Lauer\\Universit\'e de Lorraine, CNRS, LORIA, F-54000 Nancy, France}
\date{} 

\maketitle

\begin{abstract}
This paper proposes a simple approach to derive efficient error bounds for learning multiple components with sparsity-inducing regularization. We show that for such regularization schemes, known decompositions of the Rademacher complexity over the components can be used in a more efficient manner to result in tighter bounds without too much effort. We give examples of application to switching regression and center-based clustering/vector quantization. Then, the complete workflow is illustrated on the problem of subspace clustering, for which decomposition results were not previously available. For all these problems, the proposed approach yields risk bounds with mild dependencies on the number of components and completely removes this dependence for nonconvex regularization schemes that could not be handled by previous methods. 
\end{abstract}

\section{INTRODUCTION}

This paper focuses on learning problems involving multiple components. A good example is vector quantization (or center-based clustering), in which one is interested in estimating a model (or codebook) made of a finite number of components (or codepoints) that can well approximate the observations of a random variable. Other examples include subspace clustering, where the data points are approximated by a collection of subspaces rather than codepoints, and switching regression, that works similarly but with random input--output pairs and components that are functions approximating the output given the input. 
In this paper, we propose a unified approach to derive generalization error bounds for all these problems which yields bounds with a mild dependence on the number of components for classes of interdependent components.  
While generalization might not be the primary goal in these problems, such error bounds can lead to model selection strategies and have been the subject of many studies, see, e.g., \cite{Bartlett98b,Biau08,Lauer18} and references therein.

More precisely, we show how to efficiently take into account the invariance of the loss with respect to permutations of the components to derive risk bounds in multiple component learning problems. The proposed approach is simple and applies to different problems merely by plugging known decomposition results for these problems. For products of independent component classes, a {\em decomposition result} is one that decomposes the Rademacher complexity of the loss class into a sum of Rademacher complexities over the component classes. Previous works used such decompositions to obtain risk bounds that grow linearly with the number $C$ of components.  
But for classes constrained in terms of a sparsity-inducing complexity measure, such as an $\ell_p$-norm over the complexities of the components, our approach yields risk bounds with a dependence on $C$ that varies for instance between $O(\sqrt{C})$ for $p=2$ and $O(\log C)$ for $p=1$. Such sparsity-inducing regularization schemes were already considered by \cite{Lei15,Maurer16}, where similar dependencies on the number of categories were obtained for multi-class classification. However, the method of \cite{Lei15,Maurer16} relies on more complex arguments involving structural results on Rademacher and Gaussian complexities, duality, strong convexity and other tools developed by \cite{Kakade12}. Here, we develop the approach in Sect.~\ref{sec:general} in a few lines with simple arguments and without invoking other tools. 
In addition, the proposed method also allows for the use of nonconvex regularization by $\ell_p$-quasi-norms with $p\in(0,1)$, which favors even sparser models. While the analysis of \cite{Lei15,Maurer16} was limited to $p\geq 1$ and a logarithmic dependence on $C$, our approach completely removes the dependence on $C$ for nonconvex regularization with $p<1$.
  
In Sect.~\ref{sec:switched}, we apply our approach to switching regression, i.e., the problem of learning a collection of regression models from a mixed data set. For sparsity-inducing regularization schemes, this allows us to tighten the bounds of \cite{Lauer18} from a linear dependence on $C$ to the ones discussed above for the different values of~$p$. Similar results are obtained in Sect.~\ref{sec:clustering} for vector quantization/clustering in Hilbert space, for which the literature only provides error bounds with either a radical dependence on $C$ in the finite-dimensional case \citep{Bartlett98b} or a linear one for infinite-dimensional Hilbert spaces \citep{Biau08}. 
Finally, Section~\ref{sec:subspace} is dedicated to the subspace clustering problem, which has a lot of applications in computer vision, for instance for motion segmentation or face clustering \citep{Vidal11,Vidal16book}, but has not yet received much attention from the viewpoint of learning theory and risk bounds. This offers us the opportunity to illustrate the complete workflow for the application of the proposed approach.

Technically, our bounds are based on the analysis of the Rademacher complexity of the loss class to derive uniform risk bounds. More advanced tools, such as those of \cite{Bartlett05} or \cite{Mendelson14}, could be used to derive bounds with faster convergence rates or even for unbounded variables. However, these tools are particularly efficient to bound the risk of the empirical risk minimizer, which, for all the multiple component learning problems we consider, cannot be easily computed (and there is no satisfactory convex surrogate loss whose minimizer could be analyzed instead). Therefore, we must focus on {\em uniform} error bounds in order to apply them to the models returned by practical algorithms.

{\em Notation.} We use $[C]=\{1,\dots,C\}$ to refer to the set of integers from 1 to $C$. Matrices are written in bold and uppercase letters, while vectors are in non-bold lowercase letters. Random variables are written in uppercase letters. Thus, $X$ will refer to a random vector, while $\g X$ is a matrix. The identity matrix is denoted by $\g I$. The Frobenius norm $\|\g A\|_F$ of a matrix $\g A\in\R^{m\times n}$ of entries $A_{ij}$ is defined as $\|\g A\|_F=\sqrt{\sum_{i=1}^m\sum_{j=1}^n A_{ij}^2}$. $\trace(\g A)$ denotes the trace of the matrix $\g A$ and we have $\|\g A\|_F = \sqrt{\trace(\g A^\top\g A)}$. For a vector $a\in\R^C$ and any $p\in(0,\infty)$, $\|a\|_p=\left(\sum_{k=1}^C |a_k|^p\right)^{1/p}$ denotes its $\ell_p$-norm for $p\geq 1$ or $\ell_p$-quasi-norm for $p\in(0,1)$, while $\|a\|_{\infty}=\max_{k\in[C]} |a_k|$ is its $\ell_{\infty}$-norm. Given two sets, $\X$ and $\Y$, $\Y^{\X}$ stands for the set of functions from $\X$ into $\Y$.

\section{GENERAL APPROACH}
\label{sec:general}

We focus on learning problems in which the aim is to learn $C\geq 2$ components from a set $\V$ on the basis of data points $z_i \in \Z$, $i=1,\dots n$. In the following, $\Z$ will be instantiated either as $\X\times \Y$ for problems with input space $\X$ and output space $\Y$ or just as $\X$ in contexts without outputs. 

Specifically, 
let $Z$ be a random variable taking values in $\Z$. 
A particular problem is characterized by a loss functional $\ell : \V^C\times \Z$, which measures the pointwise error of a model $f=(f_k)_{1\leq k\leq C}$ made of $C$ components $f_k$ from $\V$. Then, the aim is to minimize, over a predefined model class $\F\subset\V^C$, the risk 
\begin{equation}\label{eq:risk}
	L(f) = \E \ell(f, Z)  
\end{equation}
on the basis of a sample of $n$ independent copies $Z_i$ of $Z$. In particular, we concentrate on the standard strategy that minimizes the empirical risk 
\begin{equation}\label{eq:empiricalrisk}
	\hat{L}_n(f)= \frac{1}{n}\sum_{i=1}^n \ell(f, Z_i) . 
\end{equation}
However, we here focus on statistical aspects of learning and will not discuss algorithmic issues related to the actual minimization of this quantity, which can be highly nontrivial \citep{Aloise09,Lauer15b}. Instead, we will particularly pay  attention to the derivation of upper bounds on the risk that hold {\em uniformly} over the class $\F$, and thus not only for the empirical risk minimizer which remains elusive in many practical cases. 

Before we expose our approach to the derivation of such bounds, we first give a few definitions and start with the one that characterizes the losses considered in this paper.

\begin{definition}[Permutation-invariant loss]\label{def:permutationinvariant}
A {\em permutation-invariant loss} over $C$ components from a set $\mathcal{V}$ is a loss functional $\ell : \mathcal{V}^C\times\Z$ such that, for any permutation $(l(k))_{1\leq k\leq C}$ of $[C]$, any $f=(f_k)_{1\leq k\leq C}\in\V^C$ and any $z\in\Z$, 
$$
	\ell( f, z) = \ell( (f_{l(k)})_{1\leq k\leq C}, z).
$$ 
\end{definition}

\begin{definition}[Loss class]\label{def:lossclass}
Given a bounded loss $\ell:\V^C\times\Z \to [0,M]$ and a class $\F\subset\V^C$, the loss class induced by $\F$ is
$$
	\L_{\F} = \left\{ \ell_f \in[0,M]^{\Z} : \ell_f(z) = \ell(f,z),\ f\in\F\right\}. 
$$
\end{definition}

\begin{definition}[Rademacher complexities]
\label{def:rad}
Let $T$ be a random variable with values in $\mathcal{T}$. 
For $n \in \mathbb{N}^*$,
let $\g{T}_n = \left( T_i  \right)_{1 \leq i \leq n}$
be an $n$-sample of independent copies of $T$, let
$\boldsymbol{\sigma}_n = \left ( \sigma_i \right )_{1 \leq i \leq n}$
be a sequence of independent random variables uniformly distributed in $\{-1,+1\}$. 
Let $\mathcal{F}$ be a class of real-valued functions with domain $\mathcal{T}$.
The {\em empirical Rademacher complexity} of $\mathcal{F}$ given $\g T_n=\g t_n=\left( t_i  \right)_{1 \leq i \leq n}$ is
$$
\hat{\mathcal{R}}_n \left ( \mathcal{F} \right ) 
= \mathbb{E} 
\sup_{f \in \mathcal{F}} \frac{1}{n}
\sum_{i=1}^n \sigma_i f \left ( t_i \right ),
$$
and its {\em Rademacher complexity}, $\mathcal{R}_n \left ( \mathcal{F} \right ) =\E\hat{\mathcal{R}}_n \left ( \mathcal{F} \right ) $, is obtained by taking the expectation wrt. $\g T_n$. 
\end{definition}

The regularization schemes for learning multiple components that we consider are based on two levels of complexity measures. On the first level, let $\omega : \V\to [0,+\infty)$ be a complexity measure for a single component from $\V$ and, for any model $f\in\V^C$, let $\Omega(f) = \left(\omega(f_k)\right)_{1\leq k\leq C}$ denote the vector of $\R^C$ obtained by a component-wise application of $\omega$ to $f$. Then, at a second level, we measure the complexity of the overall model $f$ by the $\ell_p$-(quasi-)norm of $\Omega(f)$. 
Therefore, in this paper we will focus on the derivation of error bounds for classes
\begin{equation}\label{eq:FOmega}
	\F = \left\{f\in\V^C  :  \|\Omega(f)\|_p \leq \Lambda \right\}. 
\end{equation}

\begin{definition}[Ordered class $\tilde{\F}$]\label{def:orderedclass}
Given a complexity measure $\omega$ as defined above, we denote by $\tilde{f}$ an {\em ordered version} of $f\in\V^C$ with its components ordered in decreasing order of their complexity: 
$$
	\forall f=(f_k)_{1\leq k\leq C}\in\V^C,\qquad \tilde{f} = (f_{l(k)})_{1\leq k\leq C} ,
$$ 
where $l(k)$ is the $k$th element of a permutation of $[C]$ that ensures 
$$
	\omega(\tilde{f}_1) \geq \dots\geq \omega(\tilde{f}_C).
$$
Then, for any class $\F \subset\V^C$, the {\em ordered class} $\tilde{\F}$ is defined by reordering the elements of $\F$:  
$$
	\tilde{\F} = \left\{ \tilde{f} : f\in\F\right\} .
$$
\end{definition}

Note that for classes built as $\F=\F_0^C=\F_0\times\dots\times\F_0$ for some $\F_0\subset\V$, the ordered class $\tilde{\F}$ is a subset of $\F$: $\forall f\in\F_0^C,\quad \tilde{f}\in\F_0^C$. This is also true for classes $\F$ as in~\eqref{eq:FOmega},
which introduce a dependence between components, encoded in the choice of $\ell_p$-(quasi-)norm. For instance, if we let $p=\infty$, then $\F$ in~\eqref{eq:FOmega} can be written as a product of independent component classes: 
\begin{align}\label{eq:Finf}
	\F_{\infty} &= \left\{f\in\V^C  :  \max_{k\in[C]}\omega(f_k) \leq \Lambda \right\} \\
	&= \prod_{k=1}^C \left\{ f_k\in\V : \omega(f_k)\leq \Lambda \right\} . \nonumber
\end{align}
But if we consider $p\in(0,\infty)$, then $\F$ in~\eqref{eq:FOmega} cannot be written as a mere product, since the complexity $\omega(f_k)$ influences the range of values allowed for $\omega(f_j)$, $j\neq k$. 
For such classes, the ordered class is a strict subset of $\F$: 
$\tilde{\F} \subset \F$, and $\tilde{\F} \neq \F$. 
The inclusion results from the permutation-invariance of the $\ell_p$-norm: $\|\Omega(\tilde{f})\|_p = \|\Omega(f)\|_p \leq \Lambda$; and this also implies that there are some $f\in\F$ with $\omega(f_2)>\omega(f_1)$ and thus that do not belong to $\tilde{\F}$. 

The interest of the ordered class $\tilde{\F}$ and the fact that it is a subset of $\F$ is highlighted by the following, which shows that for permutation-invariant losses, we can restrict the analysis to this subset of $\F$.
\begin{lemma}\label{lem:general}
Given a bounded permutation-invariant loss $\ell:\V^C\times \Z\to [0,M]$ and a class $\F\subset\V^C$, the risk of any $f\in\F$ can be bounded in terms of the Rademacher complexity of the loss class induced by the ordered class $\tilde{\F}$ instead of $\F$, namely, each of the following holds with probability at least $1-\delta$:
$$
	\forall f\in\F,\quad L(f) \leq \hat{L}_n(f) + 2\mathcal{R}_n(\L_{\tilde{\F}}) + M\sqrt{\frac{\log \frac{1}{\delta}}{2n}},
$$
$$
	\forall f\in\F,\quad L(f) \leq \hat{L}_n(f) + 2\hat{\mathcal{R}}_n(\L_{\tilde{\F}}) + 3M\sqrt{\frac{\log \frac{2}{\delta}}{2n}}.
$$
\end{lemma}
\begin{proof}
By Definitions~\ref{def:permutationinvariant} and~\ref{def:orderedclass}, $L(f) = L(\tilde{f})$ and $\hat{L}_n(f)=\hat{L}_n(\tilde{f})$. Therefore, the lemma is just a direct consequence of standard error bounds, e.g., Theorem 3.1 in \cite{Mohri12}, holding uniformly over the ordered class $\tilde{\F}$ instead of $\F$. 
\end{proof}

For classes $\F$ as in~\eqref{eq:FOmega} with dependent components, a second interest lies in the fact that $\tilde{\F}$ can be easily embedded in a product of independent component classes with decreasing complexity:
\begin{lemma}\label{lem:embedding}
Let $\F$ be as in~\eqref{eq:FOmega} with $p\in(0,\infty]$. Then,
$$
	\tilde{\F} \subseteq \Pi_p = \prod_{k=1}^C \left\{ f_k\in \V : \omega(f_{k}) \leq k^{-\frac{1}{p}} \Lambda \right\}.
$$	
\end{lemma}
\begin{proof}
Assume $p<\infty$ (see \eqref{eq:Finf} for the case $p=\infty$). Then, with $\F$ as in~\eqref{eq:FOmega}, the permutation-invariance of the $\ell_p$-norm implies that, for all $f\in\F$, 
$$
	\|\Omega(\tilde{f})\|_p^p = \|\Omega(f)\|_p^p \leq \Lambda^p ,
$$ 
while, for any $k\in[C]$, 
$$
	\|\Omega(\tilde{f})\|_p^p = \sum_{l=1}^C \omega(\tilde{f}_l)^p  \geq \sum_{l=1}^k \omega(\tilde{f}_l)^p \geq k \omega(\tilde{f}_k)^p,
$$ 
where the last inequality is due to the ordering of the $\tilde{f}_k$'s in Def.~\ref{def:orderedclass}.
Therefore, for all $\tilde{f}\in\tilde{\F}$ and all $k\in[C]$, 
$$
	 \omega(\tilde{f}_k)^p \leq \frac{\Lambda^p}{k},
$$
which proves the claimed set inclusion.  
\end{proof}
Note that for $p=\infty$ the product class $\Pi_p$ in Lemma~\ref{lem:embedding} is exactly $\F$ due to~\eqref{eq:Finf}, whereas for all finite $p$, $\Pi_p$ is strictly larger than $\F$ and thus $\tilde{\F}$: there exist $f$ in the product of component classes with $\sum_{k=1}^C \omega(f_k)^p > \Lambda^p$ that are thus not in $\F$ and not in $\tilde{\F}\subset\F$. Therefore, the inclusion provided by Lemma~\ref{lem:embedding} is not tight, but its interest lies at another level, namely, the fact that decomposition results available for products of independent classes can help us to bound the Rademacher complexity of $\L_{\tilde{\F}}$. 

Instead of deriving a generic framework with cumbersome notations that would encompass many different settings but would also hide the simplicity of the approach, the following illustrates the application of the method on a few examples. In particular, we detail below the settings of switching regression and center-based clustering, for which decomposition results can be found in the literature. 
Then, we will show in Sect.~\ref{sec:subspace} how to develop the complete workflow for subspace clustering from the definition of the loss function to the derivation of efficient bounds, including the obtention of a decomposition result.  

For all these settings we shall derive error bounds with a dependence on $C$ characterized by $p$ via the function 
\begin{equation}\label{eq:alpha}
	\alpha(C,p) =\begin{cases} 
				C, & \mbox{if } p=\infty\\
				\frac{p}{p-1} C^{1-1/p} ,& \mbox{if } 1<p<\infty \\
				1+\log C ,  & \mbox{if } p=1\\
				\frac{1}{1-p} ,  
				& \mbox{if } 0<p<1.
			\end{cases}
\end{equation}
In particular,  the dependence on $C$ will be linear for $p=\infty$ (the case of independent component classes), radical for $p=2$ (the most common case), logarithmic for $p=1$ (a common choice for sparsity-inducing regularization) and bounded by a constant for $p<1$ (corresponding to nonconvex regularizers). 

\section{SWITCHING REGRESSION}
\label{sec:switched}

In a regression problem, one must learn a model that can accurately predict the real output $Y\in\Y\subset\R$ given the input $X\in\X$. Switching regression refers to the specific case where the process generating $Y$ can arbitrarily switch between different behaviors. The difficulty then comes from the fact that the switchings are not observed and the association of the data points $(x_i,y_i)\in\Z=\X\times \Y$ to these behaviors is unknown. Thus, the aim is to learn a collection of functions $f_k : \X\to \R$ from a mixed training sample including examples from multiple sources. An important application is that of switched system identification in control theory, see \cite{Paoletti07,LauerBook} for an overview. 
  
In such a context, the goal is to find $f\in(\R^{\X})^C$ so that at least one of its components can accurately estimate the output $Y$ given $X$. The loss can thus be defined on the basis of
$$
	\min_{k\in[C]} (y - f_k(x))^2.
$$
More precisely, we assume that $\Y$ is bounded and, without loss of generality, that $\Y=[-1/2,1/2]$. Thus, we can clip the outputs of the components at $1/2$ without increasing the error and compute the loss  with respect to the clipped functions as in \cite{Lauer18}:
\begin{equation}\label{eq:lossswitched}
	\ell(f,x,y) = \min_{k\in[C]}\left(y -  \min\left\{\frac{1}{2}, \max\left\{\frac{-1}{2} , f_k(x)\right\}\right\} \right)^2 .
\end{equation}
This ensures that the loss is bounded by $1$ for all $y\in\Y$. In addition, it is easy to see that this loss remains permutation-invariant in the sense of Definition~\ref{def:permutationinvariant}.

Here, we focus on kernel machines and consider models with components from a reproducing kernel Hilbert space (RKHS) $\H\subset \R^{\X}$ of reproducing kernel $K$ (see \cite{Berlinet04} for details). Thus, we set $\V=\H$ and the complexity measure $\omega$ to the RKHS norm $\|\cdot\|$ in the approach described above, which yields the risk bound in Theorem~\ref{thm:switched} below for classes regularized by $\|\Omega(f)\|_p = (\sum_{k=1}^C \|f_k\|^p)^{1/p}$. 

\begin{theorem}\label{thm:switched}
Let $\F= \left\{f\in\H^C : \left\|\begin{bmatrix}\|f_1\|& \dots & \|f_C\|\end{bmatrix}\right\|_p \leq \Lambda\right\}$ 
and $\alpha(C,p)$ be as in~\eqref{eq:alpha}. Then, with probability at least $1-\delta$ on the random draw of the training sample $(Z_i)_{1\leq i\leq n}=\left((X_i,Y_i)\right)_{1\leq i\leq n}$, the switching regression risk based on the loss~\eqref{eq:lossswitched} is uniformly bounded for all $f\in\F$ by
$$
	 L(f)\!\leq\!\hat{L}_n(f) + 4\alpha(C,p)\frac{\Lambda\!\sqrt{\sum_{i=1}^n\! K(X_i,X_i)}}{n} + 3\sqrt{\frac{\log\frac{2}{\delta}}{2n}} .
$$
\end{theorem}
\begin{proof}
By the permutation-invariance of $\ell$ in~\eqref{eq:lossswitched}, we can apply Lemma~\ref{lem:general} and the result follows from the computation of the (empirical) Rademacher complexity of $\L_{\tilde{\F}}$. Then, Lemma~\ref{lem:embedding} gives $\tilde{\F}\subseteq \Pi_p$ and thus $\L_{\tilde{\F}}\subseteq \L_{\Pi_p}$, which further yields
$$
	\hat{\mathcal{R}}_n(\L_{\tilde{\F}})\leq \hat{\mathcal{R}}_n(\L_{\Pi_p}).
$$
Since $\Pi_p$ is a product of independent component classes, the decomposition result in Theorem~3 of \cite{Lauer18} then gives
$$
	\hat{\mathcal{R}}_n(\L_{\tilde{\F}})
	\leq 2\sum_{k=1}^C \hat{\mathcal{R}}_n\left(\left\{ f_k\in \H : \|f_{k}\| \leq k^{-\frac{1}{p}} \Lambda \right\}\right) , 
$$
while standard computations for RKHS balls (see, e.g., \cite{Bartlett02}) further ensure that 
$$
	\hat{\mathcal{R}}_n(\L_{\tilde{\F}}) \leq \frac{2\Lambda \sqrt{\sum_{i=1}^n K(X_i,X_i)}}{n}\sum_{k=1}^C  k^{-\frac{1}{p}}  .
$$
Thus, the theorem is proved after a straightforward check that $\sum_{k=1}^C  k^{-\frac{1}{p}} \leq \alpha(C,p)$ holds for all $C\geq 2$ and $p\in (0,\infty]$ (see Appendix~\ref{app:boundsalpha} for details). 
\end{proof}

For independent component classes ($p=\infty$), this result coincides with that in Eq.~(18) of \cite{Lauer18}. However, for $p<\infty$, the dependence on $C$ improves according to the definition of $\alpha(C,p)$ in~\eqref{eq:alpha}. In particular, a radical dependence is obtained for $p=2$, which could only be obtained in \cite{Lauer18} through covering numbers and a loss in the order of $\log^{3/2} n$ in terms of convergence rate. In addition, the dependence on $C$ further improves for smaller values of $p$.

\section{VECTOR QUANTIZATION/CLUSTERING}
\label{sec:clustering}

Let $\X$ be a Hilbert space and  $\|\cdot\|$ denote its norm. 
The aim of vector quantization, as described by \cite{Bartlett98b}, is to learn a subset $\{f_k\}_{k=1}^C\subset \X$ of $C$ elements from $\X$, called codepoints, that can well represent the observations of the random variable $X\in\X$. 
Specifically, we can limit the analysis to nearest neighbors quantizers, for which the error of a model $f=(f_k)_{1\leq k\leq C}$ is measured via the loss 
\begin{equation}\label{eq:lossclustering}
	\ell(f,x) = \min_{k\in[C]} \|x - f_k \|^2 .
\end{equation}
Then, the quantity~\eqref{eq:risk} (with $Z=X$) is known as the {\em distortion} of $f$ for which upper bounds are of primary importance. 

This problem can also be seen as a center-based clustering one, in which the goal is to divide the observations of $X$ into $C$ groups centered at the $f_k$'s by minimizing the empirical risk~\eqref{eq:empiricalrisk} based on~\eqref{eq:lossclustering}. By considering the Vorono\"i partition of $\X$ associated to these centers, \cite{Biau08} interpret the quantity~\eqref{eq:risk} as the {\em clustering risk} measuring the performance of a particular model $f\in\X^C$.

The setting just described enters our framework in a straightforward manner with $\V=\Z=\X$ and $\omega=\|\cdot\|$. We can thus easily obtain efficient bounds on the clustering risk for regularized classes on the basis of the results of \cite{Biau08}.

\begin{theorem}\label{thm:clustering}
Let $X\in\X$ be such that $P(\|X\|\leq \Lambda_x)=1$, $\F= \{f\in\X^C : \left\|\begin{bmatrix}\|f_1\|& \dots & \|f_C\|\end{bmatrix}\right\|_p \leq \Lambda\}$ and $\alpha(C,p)$ be as in~\eqref{eq:alpha}. Then, with probability at least $1-\delta$ on the random draw of the training sample $\left(X_i\right)_{1\leq i\leq n}$, the clustering risk based on the loss~\eqref{eq:lossclustering} is uniformly bounded for all $f\in\F$ by
\begin{align*}
	L(f) \leq  \hat{L}_n(f) + 2\alpha(C,p)\left(\frac{2\Lambda\sqrt{\sum_{i=1}^n \|X_i\|^2}}{n} + \frac{\Lambda^2}{\sqrt{n}}\right)  + 3(\Lambda_x^2+\Lambda^2)\sqrt{\frac{\log\frac{2}{\delta}}{2n}} .
\end{align*}
\end{theorem}
\begin{proof}
It is easy to see that the clustering loss~\eqref{eq:lossclustering} is permutation-invariant in the sense of Definition~\ref{def:permutationinvariant} and uniformly bounded by $M=\Lambda_x^2+\Lambda^2$. Thus, as for the switching regression case, we can apply Lemmas~\ref{lem:general} and~\ref{lem:embedding}. Then, it remains only to show that $\hat{\mathcal{R}}_n(\L_{\Pi_p})$ is smaller than $\sum_{k=1}^C  k^{-\frac{1}{p}}$ times  a term independent of $k$, $C$ and $p$, in order to conclude with the use of $\sum_{k=1}^C  k^{-\frac{1}{p}} \leq \alpha(C,p)$ (see Appendix~\ref{app:boundsalpha}). 

This can be done by following the proof of Theorem~2.1 in \cite{Biau08}, which includes both a decomposition result and the computation of the Rademacher complexity of the loss class for Hilbert space balls. In fact, the statements in \cite{Biau08} do not concern the empirical version of Rademacher complexity and focus on products of similar classes so that the result is $C$ times the Rademacher complexity wrt. a single component. However, \cite{Biau08} give all the ingredients to obtain the result in the form stated here. For completeness, we give the details in Appendix~\ref{app:clusteringproof}, which lead to
\begin{align}\label{eq:clusteringproofbiau}
	\hat{\mathcal{R}}_n(\L_{\Pi_p}) & \leq \sum_{k=1}^C \left(\frac{2k^{-\frac{1}{p}}\Lambda\sqrt{\sum_{i=1}^n \|X_i\|^2}}{n} + \frac{k^{-\frac{2}{p}}\Lambda^2}{\sqrt{n}}\right) 
	\\ &\leq  \left(\frac{2\Lambda\sqrt{\sum_{i=1}^n \|X_i\|^2}}{n} + \frac{\Lambda^2}{\sqrt{n}}\right) \sum_{k=1}^C k^{-\frac{1}{p}}.\label{eq:clusteringproofbiau2}
\end{align}
\end{proof}

As for switching regression, this result encompasses for $p=\infty$ the case of independent component classes found in \cite{Biau08}. For $p<\infty$, the improved bound could also have been obtained by following the approach of \cite{Lei15} or \cite{Maurer16}, which is also very efficient for regularized classes constrained by $\sum_{k=1}^C \|f_k\|^p \leq \Lambda^p$. However, as highlighted in the introduction, this would have required $p\geq 1$ and a much heavier machinery, whereas our approach remains simple and provides a proof of Theorem~\ref{thm:clustering} also valid for nonconvex regularizers with $p\in(0,1)$ as an almost direct consequence of previous decomposition results.

\section{SUBSPACE CLUSTERING} 
\label{sec:subspace}

Subspace clustering differs from center-based clustering in that the components $f_k$ are subspaces of $\X$ instead of points. In the following, we drop the notation $f_k$ and instead focus on the subspace basis in the form of matrices $\g B_k\in\R^{d\times d_k}$. 

Our starting point in Sect.~\ref{sec:subspace1} is a uniform bound on the error when learning a single subspace. Then, we extend this to multiple subspaces in Sect.~\ref{sec:subspaceclustering} and finally tighten the bound for classes defined by $\ell_p$-norm regularization in Sect.~\ref{sec:tightsubspace}.

\subsection{Uniform Error Bounds for Subspace Estimation}
\label{sec:subspace1}

A $d_1$-dimensional subspace of $\R^d$ can be represented by a basis $\{b_1, \dots, b_{d_1}\}\subset\R^d$, i.e., by a matrix $\g B\in\R^{d\times d_1}$ with $\g B^\top\g B=\g I$, which yields the projection matrix $\g P=\g B\g B^\top$. Then, the approximation error incurred by the projection of a point $x$ onto the subspace is measured by the loss
\begin{equation}\label{eq:losssubspace}
	\ell (\g B, x) = \|\g P x - x\|^2 = \|\g B\g B^\top  x - x\|^2.
\end{equation}
We are interested here in bounding the expected approximation error (or risk), $L(\g B) = \E \ell(\g B,X)$,  
in terms of its empirical estimation, $\hat{L}_n(\g B)= \frac{1}{n}\sum_{i=1}^n\ell (\g B, X_i)$,  
for any distribution of $X$ over $\X = \{x \in\R^d : \|x\|\leq \Lambda_x\}$  and {\em uniformly} over the class of $d_1$-dimensional subspaces of $\R^d$ with basis in
\begin{equation}\label{eq:basisset}
	\mathcal{B} = \left\{ \g B\in \R^{d\times d_1}  : \g B^\top \g B = \g I \right\}.
\end{equation}

This can be done as follows (see Appendix~\ref{app:proofthmclustering} for the proof).

\begin{theorem}\label{thm:subspace1}
Let $X\in\R^d$ be a random vector such that $P(\|X\|\leq \Lambda_x)=1$. Then, with probability at least $1-\delta$ on the random draw of a data matrix $\g X=[X_1,\dots, X_n]\in \R^{d\times n}$ made of $n$ independent copies of $X$, for any subspace of dimension $d_1$ and any basis $\g B$ of that subspace, 
$$
	L(\g B) \leq \hat{L}_n(\g B) + 2\frac{\sqrt{d_1}\|\g X\|_F}{n} + 3\Lambda_x^2 \sqrt{\frac{\log \frac{2}{\delta}}{2n}} .
$$
\end{theorem}

Note that this bound is uniform and is of the same order as the non-uniform one obtained by \cite{ShaweTaylor05}.

\subsection{Multiple Subspace Learning/Subspace Clustering}
\label{sec:subspaceclustering}

We now consider the problem of learning multiple subspaces, represented by basis $\g B_k\in\R^{d\times d_k}$ and projection matrices $\g P_k$, $k=1,\dots,C$, to obtain an approximation of the distribution of $X$. This setting extends the vector quantization framework to models with subspace components and can be formally encoded by the loss 
\begin{equation} \label{eq:subspaceclusteringloss}
	\ell( (\g B_k)_{1\leq k\leq C}, x ) = \min_{k\in[C]} \|\g B_k\g B_k^\top x - x\|^2 . 
\end{equation}
In this context, the {\em subspace clustering risk}, $L(\g B)=\E \ell( \g B, X )$, of a collection $\g B=(\g B_k)_{1\leq k\leq C}$ of subspace basis $\g B_k$ can be bounded in terms of the sum of the square roots of the subspace dimensions as follows.
\begin{theorem} \label{thm:subspace}
Let $X\in\R^d$ be a random vector such that $P(\|X\|\leq \Lambda_x)=1$. Then, with probability at least $1-\delta$ on the random draw of a data matrix $\g X=[X_1,\dots,X_n]\in \R^{d\times n}$ made of $n$ independent copies of $X$, for any collection of basis $\g B$ of subspaces with fixed dimensions $d_k$,
$$
	L(\g B) \leq  \hat{L}_n(\g B) + 2\frac{\sum_{k=1}^C \sqrt{d_k} \|\g X\|_F}{n} + 3\Lambda_x^2 \sqrt{\frac{\log \frac{2}{\delta}}{2n}} .
$$
\end{theorem}
\begin{proof}
Define the loss class $\L_{\B}$ as in Definition~\ref{def:lossclass} from 
$$
	\B= \prod_{k=1}^C \B_k ,\quad \mbox{with } 
	\B_k = \left\{ \g B_k\in\R^{d\times d_k} : \g B_k^\top \g B_k = \g I\right\}. 
$$
Then, its complexity can be decomposed as a sum of those of classes induced by the $\B_k$'s. To see this, note that, with $\g P_k = \g B_k\g B_k^\top$, the loss can be reformulated as 
$$
	\ell( (\g B_k)_{1\leq k\leq C}, x ) 
		= \|x\|^2 - \max_{k\in[C]} \|\g P_kx\|^2 .
$$
Thus,  given $(X_i)_{1\leq i\leq n}=(x_i)_{1\leq i\leq n}$, 
\begin{align*}
	\hat{\mathcal{R}}_n (\L_{\B}) &=\E \sup_{\g B \in\B} \frac{1}{n}\sum_{i=1}^n \sigma_i \min_{k\in[C]} \|\g P_k x_i - x_i\|^2\\
		&\leq \E \frac{1}{n}\sum_{i=1}^n \sigma_i \|x_i\|^2 + \E \sup_{\g B \in\B } \frac{1}{n}\sum_{i=1}^n - \sigma_i \max_{k\in[C]} \|\g P_kx_i\|^2\\
		&=\E \sup_{\g B\in\B } \frac{1}{n}\sum_{i=1}^n \sigma_i \max_{k\in[C]} \|\g P_k x_i\|^2\\
		&\leq \sum_{k=1}^C \E \sup_{\g B_k\in\B_k} \frac{1}{n}\sum_{i=1}^n \sigma_i \|\g P_kx_i\|^2,
\end{align*}
where the third line uses $\E \frac{1}{n}\sum_{i=1}^n \sigma_i \|x_i\|^2=  \frac{1}{n}\sum_{i=1}^n \|x_i\|^2 \E\sigma_i=0$ and the fact that $\sigma_i$ and $-\sigma_i$ share the same distribution, while the last line is due to Lemma~8.1 in \cite{Mohri12}. 
Then, similar computations as in the proof of Theorem~\ref{thm:subspace1} (see Appendix~\ref{app:proofthmclustering}) give, for any $k\in[C]$, 
$$
	 \E \sup_{\g B_k\in\B_k} \frac{1}{n}\sum_{i=1}^n \sigma_i \|\g P_kx_i\|^2 \leq \frac{\sqrt{d_k} \|\g X\|_F}{n}
$$
and the result follows from the application of Theorem~3.1 in \cite{Mohri12} and the fact that the loss defined as the pointwise minimum of losses bounded by $\Lambda_x^2$ is also bounded by $\Lambda_x^2$. 
\end{proof}

Theorem~\ref{thm:subspace} applies to products of independent component classes, which here means that the dimensions of the subspaces do not depend one on the other, and yields a linear dependence on $C$. The next result below yields tighter bounds by precisely taking dependencies between the dimensions into account. 

\subsection{Tighter Bounds with  $\ell_p$-norm Regularization} 
\label{sec:tightsubspace}

We have now all the basic building blocks necessary to apply the approach of Sect.~\ref{sec:general} and produce tighter bounds for subspace clustering. 
Specifically, we set $\omega(f_k) = \sqrt{d_k}$ and focus on the set of basis collections with  $\ell_p$-norm regularization:
\begin{align*}
	\mathcal{B}_p = \left\{  \g B=(\g B_k)_{1\leq k\leq C}   : \g B_k\in \R^{d\times d_k},\g B_k^\top \g B_k = \g I, 
	  \left\|\begin{bmatrix}\sqrt{d_1} & \dots & \sqrt{d_C}\end{bmatrix}\right\|_p \leq \Lambda \right\}.
\end{align*}

\begin{theorem}
Let $X\in\R^d$ be a random vector such that $P(\|X\|\leq \Lambda_x)=1$ and $\alpha(C,p)$ be as in~\eqref{eq:alpha}. Then, with probability at least $1-\delta$ on the random draw of a data matrix $\g X=[X_1,\dots,X_n]\subset \R^{d\times n}$ made of $n$ independent copies of $X$, for any collection of subspace basis $\g B\in\B_p$,
$$
	L(\g B)\leq \hat{L}_n(\g B) + 2 \alpha(C,p)\frac{\Lambda\|\g X\|_F}{n} + 3\Lambda_x^2 \sqrt{\frac{\log \frac{2}{\delta}}{2n}} .
$$
\end{theorem}
\begin{proof}
First, note that the subspace clustering loss~\eqref{eq:subspaceclusteringloss} is permutation-invariant according to Def.~\ref{def:permutationinvariant}. Thus, Lemmas~\ref{lem:general} and~\ref{lem:embedding} apply and it remains only to bound $\hat{\mathcal{R}}_n(\L_{\Pi_p})$ with 
$$
	\Pi_p = \prod_{k=1}^C \left\{ \g B_k \in\R^{d\times d_k} : \g B_k^\top \g B_k = \g I, \ \sqrt{d_k} \leq k^{-\frac{1}{p}}\Lambda\right\} .
$$ 
Here, the proof of Theorem~\ref{thm:subspace} provides us with 
$$
	\hat{\mathcal{R}}_n(\L_{\Pi_p}) \leq  \frac{\sum_{k=1}^C\sqrt{d_k} \|\g X\|_F}{n} \leq \frac{\Lambda\|\g X\|_F}{n} \sum_{k=1}^C k^{-\frac{1}{p}}
$$
and plugging $\sum_{k=1}^C k^{-\frac{1}{p}}\leq \alpha(C,p)$ (see Appendix~\ref{app:boundsalpha}) completes the proof.
\end{proof}

Thus, we recover bounds for subspace clustering with similar dependencies on the main parameters ($C$ and $n$) as those obtained for switching regression and center-based clustering. Again, we emphasize that once a bound was found for products of independent component classes with a linear dependence on $C$ (Theorem~\ref{thm:subspace}), our approach easily yielded mild dependencies for classes with dependent components. 

\section{CONCLUSIONS}
\label{sec:conclusions}

The paper presented a simple approach to derive risk bounds with mild dependence on the number $C$ of components for classes with interdependent components. Only two ingredients are needed to obtain such results with the proposed approached: a permutation-invariant loss and a bound holding for products of independent component classes and providing a decomposition of their Rademacher complexity into a sum of the component complexities. 

Future work will consider the application of the proposed approach to other settings and permutation-invariant losses. The new bounds for subspace clustering could also lead to novel model selection strategies in order to tune the number of subspaces and their dimensions from the data.

\appendix

\section{USEFUL BOUNDS}
\label{app:boundsalpha}

We show here that, for any integer $C\geq 2$ and $p\in(0,\infty]$, with $\alpha(C,p)$ as defined in~\eqref{eq:alpha},   
$$
	\sum_{k=1}^C  k^{-\frac{1}{p}}\leq \alpha(C,p) .
$$
For $p=\infty$, we easily see that $\sum_{k=1}^C  k^{-\frac{1}{p}}=\sum_{k=1}^C 1 = C$.
For $p<\infty$, we can write 
$$
	\sum_{k=1}^C  k^{-\frac{1}{p}}  = 1 + \sum_{k=2}^C  k^{-\frac{1}{p}} \leq 1 + \int_{1}^C x^{-\frac{1}{p}} dx.
$$
Then, for $p=1$, we have 
$$
	\sum_{k=1}^C  k^{-\frac{1}{p}}  \leq 1 + \int_{1}^C \frac{1}{x} dx = 1 + \log C - \log 1 = 1+\log C,
$$
while for $p\neq 1$, we have 
$$
	\sum_{k=1}^C  k^{-\frac{1}{p}}	\leq  
	1 + \frac{p}{p-1}(C^{(p-1)/p} - 1) = \frac{pC^{1-1/p} - 1}{p-1}.
$$
So for $p>1$, we get
$$
	\sum_{k=1}^C  k^{-\frac{1}{p}}	< \frac{pC^{1-1/p}}{p-1} ,
$$
while for $p<1$, we obtain 
$$
	\sum_{k=1}^C  k^{-\frac{1}{p}} \leq \frac{1 - pC^{1-1/p} }{1-p}  
	\leq  \frac{1}{1-p}.
$$

\section{COMPLEMENTS FOR THE PROOF OF THEOREM \ref{thm:clustering}}
\label{app:clusteringproof}

We here restate the results embedded in the proof of Theorem~2.1 in \cite{Biau08} with empirical Rademacher complexities and a summation over the component classes, as used in the proof of Theorem~\ref{thm:clustering}. 
First, we reformulate the clustering loss as 
\begin{align*}
	\ell(f, x) &= \min_{k\in[C]}\|x - f_k(x)\|^2\\& = \|x\|^2+ \min_{k\in[C]} -2\inner{x}{f_k} + \|f_k\|^2,
\end{align*}
which, for $\Pi_p=\prod_{k=1}^C \Pi_{p,k}$ and given $(X_i)_{1\leq i\leq n}=(x_i)_{1\leq i\leq n}$, leads to
\begin{align*}
	\hat{\mathcal{R}}_n \left( \L_{\Pi_p} \right)
	 &= \E\sup_{f\in\Pi_p}\frac{1}{n} \sum_{i=1}^n \sigma_i \left(\|x_i\|^2+ \min_{k\in[C]} -2\inner{x_i}{f_k} + \|f_k\|^2 \right)\\
	&\leq  \E \frac{1}{n} \sum_{i=1}^n \sigma_i \|x_i\|^2  + \E\sup_{f\in\Pi_p}\frac{1}{n} \sum_{i=1}^n \sigma_i \min_{k\in[C]} -2\inner{x_i}{f_k} + \|f_k\|^2 \\
	&= \E\sup_{f\in\Pi_p}\frac{1}{n} \sum_{i=1}^n \sigma_i \max_{k\in[C]} 2\inner{x_i}{f_k} - \|f_k\|^2 \\
	&\leq \sum_{k=1}^C\E\sup_{f_k\in\Pi_{p,k}}\frac{1}{n} \sum_{i=1}^n \sigma_i ( 2\inner{x_i}{f_k} - \|f_k\|^2 ) ,
\end{align*}
where the last line is due to Lemma~8.1 in \cite{Mohri12}. 
Then, with $\Lambda_k=k^{-1/p}\Lambda$, for any $k\in[C]$,
\begin{align*}
	\E\sup_{f_k\in\Pi_{p,k}}\frac{1}{n} \sum_{i=1}^n \sigma_i ( 2\inner{x_i}{f_k} - \|f_k\|^2 ) 
	& \leq 	2\E\!\sup_{f_k\in\Pi_{p,k}}\!\frac{1}{n} \sum_{i=1}^n \sigma_i \inner{x_i}{f_k} + \E\!\sup_{f_k\in\Pi_{p,k}}\!\frac{1}{n} \sum_{i=1}^n \sigma_i  \|f_k\|^2 \\
		&\leq 2 \E\sup_{f_k\in\Pi_{p,k}}\frac{1}{n} \inner{\sum_{i=1}^n \sigma_i x_i}{f_k} + \frac{\Lambda_k^2}{\sqrt{n}}\\
		&\leq 2 \frac{\Lambda_k}{n} \E\left\|\sum_{i=1}^n \sigma_i x_i\right\| + \frac{\Lambda_k^2}{\sqrt{n}}
		\leq 2 \frac{\Lambda_k}{n} \sqrt{\sum_{i=1}^n \|x_i\|^2} + \frac{\Lambda_k^2}{\sqrt{n}}		.
\end{align*}
The second inequality, i.e.,~\eqref{eq:clusteringproofbiau2}, is merely due to the fact that $k^{-2/p}\leq k^{-1/p}$ for all $k\geq 1$.

\section{PROOF OF THEOREM~\ref{thm:subspace1}}
\label{app:proofthmclustering}

Since $\g P=\g B\g B^\top$ is a projection matrix, it is symmetric and idempotent: $\g P^\top \g P = \g P\g P = \g P$. Thus,
\begin{align*}
	\ell(\g B, x) &= \|\g Px - x\|^2 = x^\top \g P^\top\g Px - 2 x^\top \g Px + x^\top x \\
	&= -x^\top \g Px + \|x\|^2 = \|x\|^2 - \|\g Px\|^2.
\end{align*}
Hence, the loss is bounded with probability one as $0\leq \ell(\g B, X) \leq \|X\|^2\leq \Lambda_x^2$ and standard error bounds such as Theorem 3.1 in \cite{Mohri12} apply to the loss class based on~\eqref{eq:losssubspace} and~\eqref{eq:basisset}, 
$$
	\mathcal{L}_{\B} = \left\{ \ell \in [0,\Lambda_x^2]^{\X} : \ell(x) = \|\g B\g B^\top x - x\|^2 ,\ \g B\in\mathcal{B}\right\} .
$$
Then, the statement is a consequence of the estimation of the empirical Rademacher complexity of $\L_{\B}$ given $(X_i)_{1\leq i\leq n}=(x_i)_{1\leq i\leq n}$:
\begin{align*}
	\hat{\mathcal{R}}_n(\mathcal{L}_{\B}) &= \E \sup_{\g B\in\mathcal{B}} \frac{1}{n}\sum_{i=1}^n \sigma_i (\|x_i\|^2 -\|\g Px_i\|^2 )\\
	&\leq  \E \frac{1}{n}\sum_{i=1}^n \sigma_i \|x_i\|^2  +  \E \sup_{\g B\in\mathcal{B}} \frac{1}{n}\sum_{i=1}^n -\sigma_i \|\g Px_i\|^2 ,
\end{align*}
where $\E \frac{1}{n}\sum_{i=1}^n \sigma_i \|x_i\|^2= \frac{1}{n}\sum_{i=1}^n\|x_i\|^2 \E \sigma_i  = 0$ and $-\sigma_i$ has the same distribution has $\sigma_i$. Thus, using $ \|\g Px_i\|^2 = x_i^\top\g Px_i =  \trace(x_i^\top\g Px_i)=\trace(\g Px_ix_i^\top)$, we obtain
\begin{align*}
	\hat{\mathcal{R}}_n(\mathcal{L}_{\B}) &\leq 
	 \E \sup_{\g B\in\mathcal{B}} \frac{1}{n}\sum_{i=1}^n \sigma_i \trace(\g Px_ix_i^\top)\\
	&= \E \sup_{\g B\in\mathcal{B}} \frac{1}{n}\trace\left(\g P \left(\sum_{i=1}^n \sigma_i x_ix_i^\top\right)\right)\\
	&\leq  \E \sup_{\g B\in\mathcal{B}} \frac{1}{n}\|\g P\|_F \left\| \sum_{i=1}^n \sigma_i x_ix_i^\top \right\|_F ,
\end{align*}
where
\begin{align*}
	 \left\| \sum_{i=1}^n \sigma_i x_ix_i^\top \right\|_F^2 	
		&= \trace\left( \left(\sum_{i=1}^n \sigma_i x_ix_i^\top\right) \left(\sum_{i=1}^n \sigma_i x_ix_i^\top\right) \right)\\
		&=  \sum_{i=1}^n \sum_{j=1}^n \sigma_i\sigma_j\trace\left(  x_ix_i^\top x_jx_j^\top \right)\\
		&=  \sum_{i=1}^n \sum_{j=1}^n \sigma_i\sigma_j\trace\left(  (x_i^\top x_j)^2\right)\\
		&= \sum_{i=1}^n \sum_{j=1}^n \sigma_i\sigma_j (x_i^\top x_j)^2.
\end{align*}
In addition, since the trace of an idempotent matrix equals its rank and $\rank(\g B) = \rank(\g B\g B^\top)$, we have
\begin{align*}
	\|\g P\|_F &=   \sqrt{\trace(\g P^\top\g P)} = \sqrt{\trace(\g P)} = \sqrt{\rank(\g P)} = \sqrt{\rank(\g B)} = \sqrt{d_1}.
\end{align*}
Thus,
\begin{align*}
	\hat{\mathcal{R}}_n(\L_{\B}) &\leq   \frac{1}{n} \E\sqrt{ d_1 \sum_{i=1}^n \sum_{j=1}^n \sigma_i\sigma_j (x_i^\top x_j)^2}\\
	&\leq  \frac{1}{n} \sqrt{d_1 \sum_{i=1}^n \|x_i\|^2}  =   \frac{\sqrt{d_1} \|\g X\|_F}{n} .
\end{align*}


\end{document}